\documentclass[letterpaper, 10 pt, journal, twoside]{IEEEtran}

\usepackage{subcaption}
\usepackage{caption}
\usepackage{mathptmx} 
\usepackage{times} 
\usepackage{amsmath} 

\usepackage{amssymb}  
\usepackage{mathrsfs}                                  
\usepackage{graphics}
\usepackage{graphicx} 
\usepackage{caption}
\usepackage{float}
\usepackage{soul}
\usepackage{comment}
\usepackage{algorithmicx}
\usepackage{algpseudocode}
\usepackage[ruled,vlined]{algorithm2e}
\usepackage[english]{babel}
\usepackage{tikz}
\usepackage{ftnxtra}
\usepackage{color}
\DeclareMathAlphabet{\mathcal}{OMS}{cmsy}{m}{n}
\newtheorem{remark}{Remark}
\newtheorem{theorem}{Theorem}
\newtheorem{corollary}{Corollary}
\newtheorem{definition}{Definition}

\newtheorem{proof}{Proof}

\newif\ifdraft
\draftfalse

\title{
Convergence and Consistency Analysis for\\ A 3D Invariant-EKF SLAM 
}

\author{Teng Zhang, Kanzhi Wu, Jingwei Song, Shoudong Huang and Gamini Dissanayake
\thanks{The authors are with Center for Autonomous Systems, 
        University of Technology Sydney, Australia
        {\tt\small \{Teng.Zhang, Kanzhi.Wu, Jingwei.Song, Shoudong.Huang, Gamini.Dissanayake\}@uts.edu.au}}%
        }

\begin{document}

\maketitle

\begin{abstract}
In this paper, we investigate the convergence and consistency properties of an Invariant-Extended Kalman Filter (RI-EKF) based Simultaneous Localization and Mapping (SLAM) algorithm. Basic convergence properties of this algorithm are proven. These proofs do not require the restrictive assumption that the Jacobians of the motion and observation models need to be evaluated at the ground truth.
 It is also shown that the  output of RI-EKF is invariant under any \textit{stochastic rigid body transformation} in contrast to $\mathbb{SO}(3)$ based EKF SLAM algorithm ($\mathbb{SO}(3)$-EKF) that is only invariant under \textit{deterministic rigid body transformation}. Implications of these invariance properties on the consistency of the estimator are also discussed.  Monte Carlo simulation results demonstrate that RI-EKF outperforms $\mathbb{SO}(3)$-EKF, Robocentric-EKF and the ``First Estimates Jacobian" EKF, for 3D point feature based SLAM.
\end{abstract}

\section{INTRODUCTION}

\IEEEPARstart{E}{xtended} Kalman filter (EKF) has been used extensively in solving Simultaneous Localization and Mapping (SLAM) problem in the past.
However, a limitation of  the traditional EKF based point feature SLAM is the possible estimator inconsistency. Inconsistency here refers to the fact that the algorithm underestimates the uncertainty of the estimate leading to an overconfident result.  
This issue was  recognized as early as in 2001 \cite{Julier-2001} and then discussed in detail later in \cite{HuangT-RO2007}\cite{bailey2006consistency}. 
Some research to enhance the consistency of EKF SLAM is reported in the literature.  Robocentric EKF SLAM \cite{Castellanos04} estimates the location of landmarks in the robot local coordinate frame. As a result landmark positions to be estimated keep changing although landmarks are stationary in a fixed global coordinate frame. However, it has been shown that this robot-centric formulations lead to better performance in terms of estimator consistency. Guerreiro et al. \cite {GAS-SLAM-TRO2013}  also reported a Kalman filter for  SLAM problem  formulated in a robocentric coordinate frame. 
Besides, it was shown in \cite{GPHuang-consistency08} that the inconsistency in EKF SLAM is closely related to the partial observability of SLAM problem \cite{Andrade-Cetto-2004}\cite{Lee2006}. This insight resulted in a number of EKF SLAM algorithms which significantly improve consistency, such as the ``First Estimates Jacobian" EKF SLAM \cite{GPHuang-consistency08}, observability-constrained EKF SLAM \cite{GPHuang-consistencyIJRR}\cite{Hesch-consistencyTRO}. 

On the other hand, a number of authors have addressed the behaviour of EKF SLAM to examine the convergence properties and derive bounds for the uncertainty of the estimate. In 2001, Dissanayake et al. \cite{Dissa2001} proved three essential convergence properties of the algorithm under the assumption of linear motion and observation models, with theoretical achievable lower bounds on the resulting covariance matrix. In 2006, Mourikis and Roumeliotis \cite{Accuracy-SLAM-RSS2006} provided an analytical upper bound of the map uncertainty based on the observation noise level, the process noise level, and the size of the map. In 2007, Huang and Dissanayake \cite{HuangT-RO2007} extended the proof of the convergence properties and the achievable lower bounds on covariance matrix in \cite{Dissa2001} to the nonlinear case, but under a restrictive assumption that the Jacobians are evaluated at the ground truth.

  
Recently, Lie group representation for three-dimensional orientation/pose has become popular in SLAM. (e.g., \cite{forster2015imu}\cite{barfoot2016state}), which can achieve better convergence and accuracy for both filter based algorithms (e.g., \cite{mahony2008nonlinear}\cite{carlone2015quaternion}) and the optimization based algorithms (e.g., \cite{CoP}\cite{Udo-2013}). Besides,
the use of symmetry and Lie groups for observer design has gradually been recognized (e.g., \cite{Bonnabel2002}).
The combination of the symmetry-preserving theory and the extended Kalman filter  gives birth to the Invariant-EKF (I-EKF), which makes the traditional EKF possess the same invariance as the original system by using a geometrically adapted correction term. 
In \cite{Bonnabel2012}, the I-EKF methodology is firstly applied to EKF-SLAM. 
And then the Right Invariant Error EKF (called ``RI-EKF'' in this paper) for 2D SLAM is proposed in \cite{BarrauB15}, which also intrinsically uses the Lie group representation, and the improved consistency is proven  based on the linearized error-state model.

In this paper, we analyze the convergence  and consistency properties of RI-EKF for 3D case.  
 A convergence analysis for RI-EKF that does not require ``Jacobians evaluated at the ground truth'' assumption is presented. 
Furthermore, it is proven that the output of the filter is invariant under any  \textit{stochastic rigid body transformation} in contrast to $\mathbb{SO}(3)$ based EKF SLAM algorithm ($\mathbb{SO}(3)$-EKF) that is only invariant under \textit{deterministic rigid body transformation}. 
We also discuss the relationship between these invariance properties and consistency and show that 
these  properties have significant effect on the performance of the estimator via theoretical analysis and Monte Carlo simulations.

This paper is organized as follows. Section \ref{Section::PS} recalls the motion model and the observation model of SLAM problem in 3D. Section \ref{Section::REKF} provides the RI-EKF SLAM algorithm.  Section \ref{Section::Convergence} proves the convergence results of RI-EKF in two fundamental cases. Section \ref{Section::Consis} proves  
 the invariance property of RI-EKF and discusses the importance of this property. Section \ref{Section::Simulations}  demonstrates  RI-EKF  outperforms  $\mathbb{SO}(3)$-EKF, Robocentric-EKF  and ``First Estimates Jacobian" EKF SLAM algorithm  through Monte Carlo simulations. Finally, Section \ref{Section::Conclusion} outlines the main conclusions of this work.

{{
\textbf{Notations:} }} Throughout this paper bold lower-case and upper-case letters are reserved for vectors  and matrices/elements in manifold, respectively.
The notation ${S}(\cdot )$ is the skew symmetric operator that transforms a 3-dimensional vector into a skew symmetric matrix: ${S}(\mathbf{x})\mathbf{y}=\mathbf{x} \times \mathbf{y}$ for $\mathbf{x}, \mathbf{y} \in\mathbb{R}^3$, where the notation $\times$ refers to the cross product.

\section{Problem Statement}
\label{Section::PS}

The EKF SLAM algorithms focus on estimating the current robot pose and the positions of all the observed landmarks with the given motion model  and the observation model. In this work, SLAM problem in 3D scenarios is investigated and the state to be estimated is denoted by
\begin{equation}
\mathbf{X}=\left ( \mathbf{R}, \mathbf{p}, \mathbf{f}^1,\cdots, \mathbf{f}^N \right),
\label{eq::X}
\end{equation}
where $ \mathbf{R} \in \mathbb{SO}(3) $ and $ \mathbf{p} \in \mathbb{R}^{3}$ are the robot orientation and robot position, $ \mathbf{f}^{i} \in \mathbb{R}^{3}$ ($i=1,\cdots,N$) is the coordinate of the landmark $ i $,  all  described in the fixed world coordinate frame.

A general motion model for moving robot and static landmarks in 3D scenarios can be represented by 
\begin{equation}
\begin{aligned}
&  \mathbf{X}_{n+1}=f( \mathbf{X}_{n}, \mathbf{u}_{n}, \pmb{\epsilon}_{n} ) \\
=&   \left( \mathbf{R}_{n} \exp(\mathbf{w}_n+ \pmb{\epsilon}^w_\mathbf{n}), \mathbf{p}_n+ \mathbf{R}_n(\mathbf{v}_n + \pmb{\epsilon}^v_\mathbf{n} ),
\mathbf{f}^1_n,\cdots, \mathbf{f}^N_n \right),
\end{aligned}
\label{eq::RobotMotion}
\end{equation}
where  
$\mathbf{u}_n=\left[\begin{array}{cc}
\mathbf{w}^\intercal_n  & \mathbf{v}^\intercal_{n}
\end{array}    \right]^\intercal \in\mathbb{R}^6 $ is the odometry, being $ \mathbf{w}_{n} \in\mathbb{R}^3 $ and $ \mathbf{v}_{n} \in\mathbb{R}^3$ the angular increment  and linear translation from time $ n $ to time $ n+1 $,  $\exp(\cdot)$ is the exponential mapping of $\mathbb{SO}(3)$  defined in (\ref{eq::expS})   and  $\pmb{\epsilon}_n= \left[ \begin{array}{cc }
(\pmb{\epsilon}^w_\mathbf{n})^\intercal & (\pmb{\epsilon}^v_\mathbf{n})^\intercal  
\end{array}  \right]^\intercal \sim \mathcal{N}(\mathbf{0}, \pmb{\Phi}_n )$ is the odometry noise at time $n$.

As the robot is likely to observe different sets of landmarks in each time step, the notation $\text{O}_{n+1}$ is used to represent the set that indicates the landmarks observed at time $n+1$. Also by assuming a 3D sensor which provides the coordinate of landmark $ i $ in $ n+1 $-th robot frame, the observation model is given as follows
\begin{equation}
\mathbf{z}_{n+1}=h_{n+1}(\mathbf{X}_{n+1}, \pmb{\xi}_{n+1} ),
\label{eq::obsmodel}
\end{equation}
where $h_{n+1}(\mathbf{X}_{n+1} , \pmb{\xi}_{n+1})$ is a column vector obtained by stacking all entries $h^i (\mathbf{X}_{n+1} , \pmb{\xi}^i_{n+1})=\mathbf{R}_{n+1}^\intercal(\mathbf{f}^i_{n+1}-\mathbf{p}_{n+1}) +\pmb{\xi}^i_{n+1} \in\mathbb{R}^3 $ for all $i\in \text{O}_{n+1}$, $\pmb{\xi}_{n+1} \sim \mathcal{N}(\mathbf{0},\pmb{\Psi}_{n+1})  $ is the observation noise vector obtained by stacking all entries $\pmb{\xi}^i_{n+1} \sim \mathcal{N}(\mathbf{0},\pmb{\Psi}^i_{n+1})$ $(i\in \text{O}_{n+1})$. The covariance matrix $\pmb{\Psi}_{n+1}$ of observation noise  is a block diagonal matrix consisting of all $\pmb{\Psi}^i_{n+1}$ $(i\in \text{O}_{n+1})$.

\section{The Invariant EKF SLAM Algorithm}
\label{Section::REKF}

\begin{algorithm}[]
	\KwIn{$\hat{\mathbf{X}}_n$, $\mathbf{P}_n$, $\mathbf{u}_n$, $\mathbf{z}_{n+1}$;    }
	\KwOut{   $\hat{\mathbf{X}}_{n+1}$, $\mathbf{P}_{n+1}$;   }
	\textbf{Propagation:} \\
	$ \hat{\mathbf{X}}_{n+1|n}  \leftarrow   f( \mathbf{X}_{n}, \mathbf{u}_{n}, \mathbf{0} )$,  $\mathbf{P}_{n+1|n} \leftarrow \mathbf{F}_{n} \mathbf{P}_n  \mathbf{F}^\intercal_{n} + \mathbf{G}_n \Phi_n \mathbf{G}^\intercal_{n} $     \;
	\textbf{Update:} \\
	$ \mathbf{S}\leftarrow \mathbf{H}_{n+1} \mathbf{P}_{n+1|n}   \mathbf{H}_{n+1}^\intercal+\pmb{\Psi}_{n+1}$,   
	$\mathbf{K} \leftarrow \mathbf{P}_{n+1|n} \mathbf{H}_{n+1}^\intercal \mathbf{S}^{-1} $;\\ 
	$\mathbf{y} \leftarrow h_{n+1}(\hat{\mathbf{X}}_{n+1|n}, \mathbf{0} )-\mathbf{z}_{n+1} $; \\
	$\hat{\mathbf{X}}_{n+1} \leftarrow  \hat{\mathbf{X}}_{n+1|n} \oplus  \mathbf{ Ky}    $, ${\mathbf{P}}_{n+1} \leftarrow (\mathbf{I}-\mathbf{KH}_{n+1})\mathbf{P}_{n+1|n}$;\\
	(In RI-EKF, $(\mathbf{F}_n, \mathbf{G}_n, \mathbf{H}_n$) are given in (4) and (5))\\
	\caption{The general EKF framework (RI-EKF)}
	\label{Alg::G-EKF}
\end{algorithm}

In this section, RI-EKF based on the general EKF framework is briefly introduced. 
In the general EKF framework, the uncertainty of $\mathbf{X}$ is
described by $ \mathbf{X}=\hat{\mathbf{X}}\oplus \mathbf{e} $, where $\mathbf{e} \sim  \mathcal{N}(\mathbf{0}, \mathbf{P} ) $ is 
a white Gaussian noise vector and  $\hat{\mathbf{X}}$ is the mean estimate of $\mathbf{X}$. The  notation $\oplus$ is commonly called retraction in differentiable geometry \cite{absil2007trust} and it is  designed as a smooth mapping such that $\mathbf{X}=\mathbf{X}\oplus \mathbf{0}$ and there exists the inverse mapping $\ominus$ of $\oplus$: $\mathbf{e}=\mathbf{X}\ominus\hat{\mathbf{X}}$. 
The process of propagation and update based on the general EKF framework has been  summarized in  Alg. \ref{Alg::G-EKF}, which  is very similar to the standard EKF. Due to different uncertainty representation (compared to the standard EKF), the Jacobians of the general EKF framework in  Alg. \ref{Alg::G-EKF} are obtained by:
\ifdraft
\else
$\mathbf{F}_n=\frac{\partial  f ( \hat{\mathbf{X}}_n \oplus \mathbf{e}, \mathbf{u}_n, \mathbf{0}  ) \ominus  f ( \hat{\mathbf{X}}_n, \mathbf{u}_n, \mathbf{0}  )  }{\partial \mathbf{e}}|_{\mathbf{e}=\mathbf{0} }$,   $\mathbf{G}_n=\frac{\partial  f ( \hat{\mathbf{X}}_n , \mathbf{u}_n, \epsilon  ) \ominus  f ( \hat{\mathbf{X}}_n, \mathbf{u}_n, \mathbf{0}  )  }{\partial \epsilon}|_{\epsilon=\mathbf{0} }$,  $\mathbf{H}_{n+1}=\frac{\partial  h_{n+1} ( \hat{\mathbf{X}}_{n+1|n} \oplus \mathbf{e}, \mathbf{0}  )  }{\partial \mathbf{e}}|_{\mathbf{e}=\mathbf{0} } $.
\fi

\subsection{RI-EKF}

RI-EKF follows the general EKF framework summarized in Alg. 1. The state space of RI-EKF is modeled  as a Lie group $ \mathcal{G}(N) $. 
The  background knowledge about Lie group $ \mathcal{G}(N)$  is provided in Appendix \ref{app::liegroup}.

\subsubsection{The choice of $\oplus$}
The retraction $\oplus$ of RI-EKF is   chosen such that 
$
\mathbf{X}=  \hat{\mathbf{X}} \oplus  \mathbf{e} : = \exp(\mathbf{e}) \hat{\mathbf{X}},
\label{eq::UncertaintyDefinition2}
$
where  $\exp$ is the exponential mapping on the Lie group $\mathcal{G}(N)$\footnote{The exponential mapping is an overloaded function for Lie group and hence we also denote $\exp$ as the exponential mapping  for the Lie group $\mathcal{G}(N)$ (given in (\ref{eq::exp})). {More details and the Matlab code of the algorithms are available at ``https://github.com/RomaTeng/EKF-SLAM-on-Manifold''.}}, $\mathbf{X}\in\mathcal{G}(N)$ is the actual pose and landmarks,  $\hat{\mathbf{X}}\in \mathcal{G}(N)   $  is  the \textit{mean} estimate   and the uncertainty vector $\mathbf{e}=  \left[  \begin{array} {ccccc}
\mathbf{e}_\theta ^\intercal & \mathbf{e}_p^\intercal & (\mathbf{e}^1)^\intercal & \cdots & (\mathbf{e}^N)^\intercal
\end{array}\right]^\intercal
\in \mathbf{R}^{3N+6} $ follows the  Gaussian distribution $\mathcal{N}(\mathbf{0},\mathbf{P})$. 

\subsubsection{Jacobian matrices}
\label{Section::Update}
The Jacobians of the  propagation step of RI-EKF are
\begin{equation}
\mathbf{F}_n =\mathbf{I}_{3N+6}, \text{ } 
\mathbf{G}_n = {ad}_{\hat{\mathbf{X}}_n}\mathbf{B}_n,
\label{eq::FNGNL}
\end{equation}
where $
\mathbf{B}_{n} = 
\left[
\begin{array}{cc}
-{J}_r(- \mathbf{w}_n) & \mathbf{0}_{3,3}  \\
-{S}(\mathbf{v}_n){J}_r(-\mathbf{w}_n) &  \mathbf{I}_3 \\
\mathbf{0}_{3N, 3}   &   \mathbf{0}_{3N, 3}
\end{array}
\right]$. The adjoint operation ${ad} $ and the right Jacobian ${J}_r (\cdot)  $  are given in Appendix A.
The Jacobian matrix $\mathbf{H}_{n+1}$ of the update step is obtained by stacking all matrices $\mathbf{H}^i_{n+1}$ for all $i\in \text{O}_{n+1} $, where
\begin{equation}
\mathbf{H}_{n+1}^i= \left[ 
\begin{array}{ccccc}
\mathbf{0}_{3, 3} &\hat{\mathbf{R}}^\intercal_{n+1|n} & \cdots & -\hat{\mathbf{R}}^\intercal_{n+1|n} & \mathbf{0}_{3, 3(N-i )}  
\end{array}
\right].
\end{equation}
For a  general observation model that is  a function of the relative position of the landmark,
the Jacobian matrix $\mathbf{H}_{n+1}$ can be calculated by  the chain rule.

\subsubsection{Landmark initialization}\label{subsection::NewLandmark}
Here we provide the method to augment the state $\mathbf{\mathbf{X}}\in\mathcal{G}(N)$ and adjust the covariance matrix $\mathbf{P}$ when the robot observes a new landmark with the observation $\mathbf{z} \in \mathbf{R}^3 $.
For brevity,  the mathematical derivation is ignored here and the process to augment the state is summarized in Alg. \ref{Alg::Aug}, where
$
\mathbf{M}_{N}:=
\left[
\begin{array}{ccc}
\mathbf{0}_{3, 3}&\mathbf{I}_3& \mathbf{0}_{3, 3N}
\end{array}
\right]^{\intercal}
\label{eq::MN}
$
and $\Psi$ is the  covariance matrix representing the noise level  in  the new landmark observation.
\begin{algorithm}[]
	\KwIn{$\hat{\mathbf{X}}$, $\mathbf{P}$, $\mathbf{z}$;    }
	\KwOut{   $\hat{\mathbf{X}}_{new}$, $\mathbf{P}_{new}$;   }
	\textbf{Process:} \\
	$\hat{\mathbf{f}}^{N+1}= \hat{\mathbf{p}}+\hat{\mathbf{R}} \mathbf{z} \in \mathbb{R}^3$\\
	$\hat{\mathbf{X}}_{new} \leftarrow (\hat{\mathbf{X}},\hat{\mathbf{f}}^{N+1}) \in \mathcal{G}(N+1)$;\\
	$\mathbf{P}_{new} \leftarrow
	\left[
	\begin{array}{cc}
		\mathbf{P} & \mathbf{P}\mathbf{M}_N    \\
		\mathbf{M}_N^\intercal \mathbf{P} & \hat{\mathbf{R}} \Psi \hat{\mathbf{R}}^\intercal + \mathbf{M}_{N}^\intercal \mathbf{P} \mathbf{M}_{N}
	\end{array}
	\right]$.
	\caption{Landmark Initialization of RI-EKF}
	\label{Alg::Aug}
		\vspace*{-1mm}
\end{algorithm}

{{
\subsection{Discussion}
The general EKF framework proposed in \cite{BarrauB15} allows more flexible uncertainty representation, compared to the standard EKF. 
From Alg. 1, one can see that \textbf{a general EKF framework based filter} can be designed via \textbf{a choice of} \textbf{retraction} $\oplus$.
For example, the retraction $\oplus$ used in the 2D traditional EKF SLAM algorithm  is the standard addition.
A natural extension of the 2D traditional EKF SLAM algorithm is $\mathbb{SO}(3)$-EKF, in which the state space is modeled as 
$\mathbb{SO}(3) \times \mathbb{R}^{3+3N}$ and
the retraction $\oplus$ is
$ \mathbf{X}=\hat{\mathbf{X}} \oplus  \mathbf{e} = (\exp(\mathbf{e}_\theta)\hat{\mathbf{R}}, \mathbf{e}_p+\hat{\mathbf{p}}, \mathbf{e}^1+\hat{\mathbf{f}}^1, \cdots,  \mathbf{e}^N+\hat{\mathbf{f}}^N )$. 
Similarly, $\mathbb{SE}(3)$-EKF can be obtained by modeling  the state space  as $\mathbb{SE}(3) \times \mathbb{R}^{3N}$. 

Another noticeable point is that
two  general EKF framework based filters with different choice of $\oplus$ may have the same Jacobians ($\mathbf{F}_n,\mathbf{G}_n,\mathbf{H}_n$). For example, 
if the  retraction $\oplus$ of RI-EKF is changed such that $ \hat{\mathbf{X}} \oplus \mathbf{e}= (\exp(\mathbf{e}_\theta)\hat{\mathbf{R}}, \exp(\mathbf{e}_\theta)\hat{\mathbf{p}}+ \mathbf{e}_p, \cdots, \exp(\mathbf{e}_\theta)\hat{\mathbf{f}}^N+ \mathbf{e}^N  )  $, the resulting filter (Pseudo-RI-EKF) has the same Jacobians as that of RI-EKF but their performances are significantly different as shown in Section VI, showing that  the choice of retraction $\oplus$ has significant effect on the behavior of the general EKF framework based filter. In the next section, we will
compare the behavior of RI-EKF  with $\mathbb{SO}(3)$-EKF via the theoretical proofs for 
the convergence property of RI-EKF and two simple examples.  
}}

\section{Convergence Analysis of RI-EKF SLAM algorithm}
\label{Section::Convergence}	

The general expression for the covariance matrices evolution of RI-EKF cannot be easily obtained. Therefore, two representative scenarios are considered: (i) the robot is stationary, and (ii) the robot then moves one step. The convergence results of RI-EKF SLAM algorithm in the two scenarios are presented and proven, under the condition that Jacobians are  evaluated at the latest estimate. Hence the results are significant extension to   similar theorems in \cite{HuangT-RO2007} where Jacobians evaluated at the true state are assumed to be available. 

The general setting analyzed in the following subsections is as follows. 
The robot starts at point A with the initial condition $(\hat{\mathbf{X}}_0, \mathbf{P} ) $, where 
$\mathbf{P}$ is covariance matrix and 
  $\hat{\mathbf{X}}_0 = (\hat{\mathbf{R}},\hat{\mathbf{p}},\hat{\mathbf{f}}^1,\cdots,\hat{\mathbf{f}}^N)$ ($N$ landmarks have been observed). The covariance matrix of odometry noise is $\pmb{\Phi}$ and the covariance matrix of one landmark observation noise  is $\Psi$.
In the following subsections, $
\mathbf{M}_{N}:=
\left[
\begin{array}{ccc}
\mathbf{0}_{3, 3}&\mathbf{I}_3& \mathbf{0}_{3, 3N}
\end{array}
\right]^{\intercal}
$, 	$ \mathbf{L}:= \mathbf{P}\mathbf{M}_N $ and $ \mathbf{W}:=\mathbf{M}_N^\intercal \mathbf{P}\mathbf{M}_N   $. The odometry and the covariance  of odometry noise are $\mathbf{0}_{6,1}$ and $\mathbf{0}_{6,6}$, respectively when robot remains stationary.

\subsection{Scenario A: Robot remains stationary}
\label{section::SA}

\begin{theorem}
	If the robot remains stationary at point A and does not  observe any of the previously seen landmarks  but observes 
	a new landmark for $k$ times, the \textit{mean} estimate of robot pose and the $N$ landmarks and their related uncertainty remain unchanged (via $k$ times process of propagation and update of RI-EKF). The covariance matrix of the state when the new landmark is integrated becomes 
	$
	\mathbf{P}_k= 
	\left[
	\begin{array}{cc}
	\mathbf{P} &   \mathbf{L}    \\
	\mathbf{L}^\intercal & \frac{\hat{\mathbf{R}} \pmb{\Psi} \hat{\mathbf{R}}^\intercal}{k} + \mathbf{W} 
	\end{array}
	\right]
	$.
	When $k\rightarrow \infty$, the covariance matrix  becomes
	\begin{equation}
	\mathbf{P}^A_{\infty}= 
	\left[
	\begin{array}{cc}
	\mathbf{P} &   \mathbf{L} \\
	\mathbf{L}^\intercal  &  \mathbf{W}
	\end{array}
	\right].
	\label{eq::PAend}
	\end{equation}
\end{theorem}
\begin{proof}
	See Appendix \ref{app::theorem2}.
\end{proof}

\ifdraft
\else
\begin{figure}[htbp]
	\centering
	\includegraphics[width=0.95\linewidth]{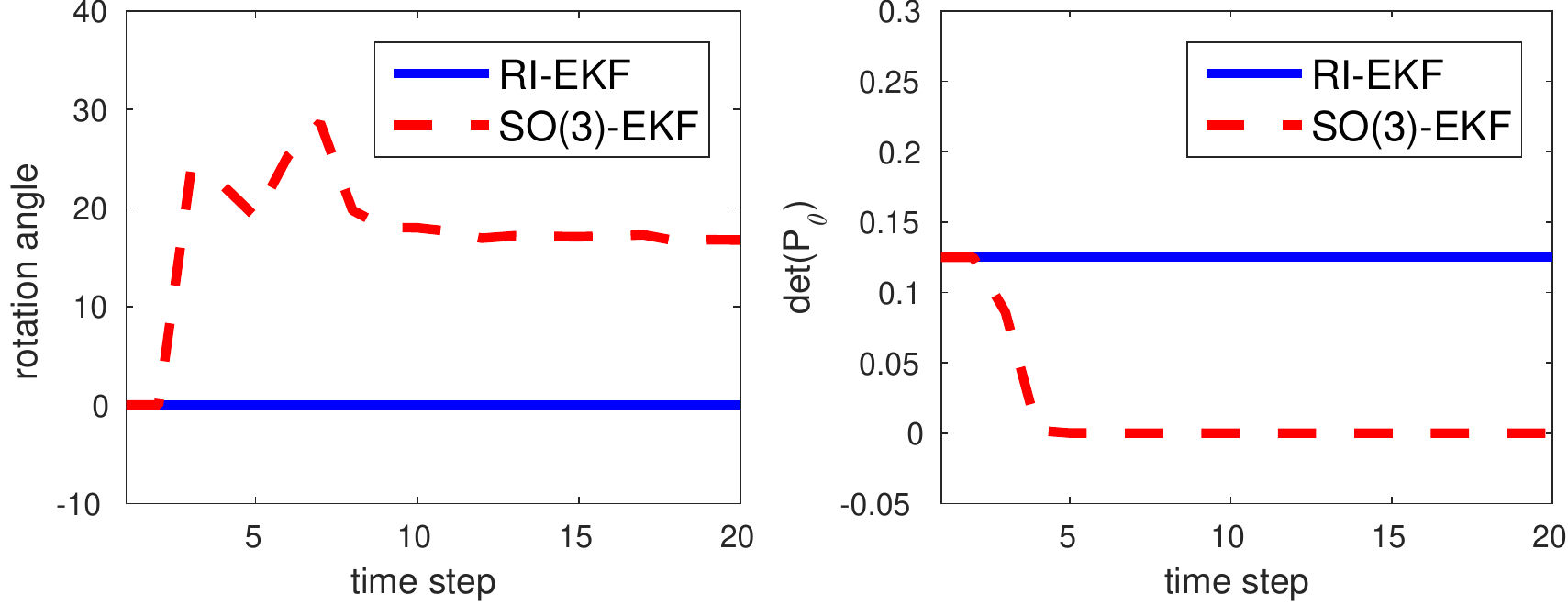}
	\caption{ \textbf{Illustration of Theorem 1}.
		In this case, robot is stationary and always only observes the ``new" landmark. 
		 Left: The error (unit: degree) in robot orientation w.r.t. the ground truth as estimated by  RI-EKF and $\mathbb{SO}(3)$-EKF.  Right:  $\det ( \mathbf{P}_{\theta} )$ estimated by RI-EKF and $\mathbb{SO}(3)$-EKF. 
		  }
	\label{fig::circ}
	\vspace*{-2mm}
\end{figure}
\fi

{ {
 Theorem 1 can be interpreted as that the covariance matrix (w.r.t. robot pose) in RI-EKF will not be reduced by observing the ``new" landmark when robot keeps stationary, which corresponds to a fact that the observations to previously unseen landmark do not convey any new information on the location of the robot. However, 
  $\mathbb{SO}(3)$-EKF does not have this good convergence property.
}}

We illustrate the results of Theorem 1 using the following scenario. The simulated robot remains stationary  and always observes the ``new" landmark (the covariance of observation noise is not zero ). The ``new" landmark is  observed multiple times (a small loop closure) and the standard  deviation of  observation noise is set as $5\%$ of robot-to-landmark distance along each
axis. The initial covariance matrix $\mathbf{P}_{\theta} \in   \mathbb{R}^{3\times3 }$   of robot orientation    is set as $\frac{1}{2}\mathbf{I}_3$. Fig. \ref{fig::circ} presents results of a simulation of this scenario. 
 The rotation angle relative to the initial orientation and $\det(\mathbf{P}_{\theta})$ from RI-EKF correctly infers that  the robot  remains stationary and the orientation uncertainty remains unchanged. In contrast, $\mathbb{SO}(3)$-EKF updates the robot orientation and furthermore predicts that the orientation uncertainty decreases as observations continue to be made, both of which are clearly erroneous and therefore leads to estimator inconsistency.

Theorem 1 can be easily extended to  a multiple landmarks scenario. 
\begin{corollary}
	If the robot is stationary at point A  and only observes $m$ new landmark $k$ times, the estimate of pose from RI-EKF does not change while the covariance matrix of the estimate becomes 
$
		\mathbf{P}_k = 
		\left[
		\begin{array}{ccccc}
		\mathbf{P} &   \mathbf{L}  &  \mathbf{L} & \cdots &  \mathbf{L}  \\
		\mathbf{L}^\intercal &    \mathbf{Q}_k & \mathbf{W} & \cdots & \mathbf{W} \\
		\mathbf{L}^\intercal&   \mathbf{W}  &  \mathbf{Q}_k  & \ddots & \vdots  \\
		\vdots & \vdots  &  \ddots  & \ddots   & \mathbf{W} \\
		\mathbf{L}^\intercal    & \mathbf{W} & \cdots  & \mathbf{W} &  \mathbf{Q}_k 
		\end{array}
		\right]
$,
	where 
	$\mathbf{Q}_k=\frac{\hat{\mathbf{R}} \pmb{\Psi} \hat{\mathbf{R}}^\intercal}{k} + \mathbf{W}$.  When $k \rightarrow \infty$, the covariance matrix becomes
	\begin{equation}
\mathbf{P}^A_{\infty}= 
\left[
\begin{array}{ccccc}
\mathbf{P} &   \mathbf{L}  &  \mathbf{L} & \cdots &  \mathbf{L}  \\
\mathbf{L}^\intercal &    \mathbf{W} & \mathbf{W} & \cdots & \mathbf{W} \\
\mathbf{L}^\intercal&   \mathbf{W}  &  \mathbf{W}  & \ddots & \vdots  \\
\vdots & \vdots  &  \ddots  & \ddots   & \mathbf{W} \\
\mathbf{L}^\intercal    & \mathbf{W} & \cdots  & \mathbf{W} &  \mathbf{W} 
\end{array}
\right].
\label{eq::MultipleAdd}
	\end{equation}
\end{corollary}

\subsection{Scenario B: Robot takes a step after a stationary period}
Consider the condition that the robot moves one step after being  stationary for a long period of time while observing new landmarks.
\begin{theorem}
	Assume $\Psi=\phi \mathbf{I}_3$ ($\phi\in\mathbb{R}^+$).
	If   the robot remains stationary at point A, does not observe any of the previously seen landmarks but observes 
	$m$ new	landmarks for $k=\infty$ times and then takes a step  to B using control action $\mathbf{u}=\left[ \begin{array}{cc}
	\mathbf{w}^\intercal & \mathbf{v}^\intercal  
	\end{array} \right]^\intercal$ and observes the same set of landmarks  $l$ times, then the covariance matrix from RI-EKF becomes
	$ \mathbf{P}^B_{l}= \mathbf{P}^A_{\infty} + \bar{\mathbf{P}}^B_l $, where $\mathbf{P}^A_{\infty}$ is given in (\ref{eq::MultipleAdd}),  $\bar{\pmb{\Psi}}=\phi \mathbf{I}_{3m}$  and 
	\begin{equation}
	\bar{\mathbf{P}}^B_{l}= \rm \textit{ad}_{\hat{\mathbf{X}}_A} \mathbf{E} ( \tilde{\pmb{\Phi}}^{-1}+\textit{l}\tilde{\mathbf{H}}^\intercal \bar{\pmb{\Psi}}^{-1}    \tilde{\mathbf{H}}   ) ^{-1} \mathbf{E}^\intercal \rm \textit{ad}_{\hat{\mathbf{X}}_A}^\intercal,
	\label{eq::pbl2}
	\end{equation}
	where $ \bar{\pmb
		{\Psi}}= \phi \mathbf{I}_{3m} $ and the covariance matrix of the odometry noise is $\pmb{\Phi}$.
	In  (\ref{eq::pbl2}), $\hat{\mathbf{X}}_A$ is the estimated state at the point A, $\tilde{\pmb{\Phi}}=   \mathbf{B} \pmb{\Phi} \mathbf{B}^\intercal  $ is a positive definite matrix and  
	\begin{equation}
	\begin{aligned}
	\mathbf{B} =&
	\left[
	\begin{array}{cc}
	- {J}_r(- \mathbf{w}) & \mathbf{0}_{3, 3}  \\
	- S(\mathbf{\mathbf{v}}) {J}_r(-\mathbf{w}) &  \mathbf{I}_3 \\
	\end{array}
	\right],\\
	\mathbf{E}=& \left[
	\begin{array}{c}
	\mathbf{I}_6\\
	\mathbf{0}_{3(N+m), 6}
	\end{array}
	\right],\  \tilde{\mathbf{H}}= \mathbf{H} {ad}_{\hat{\mathbf{X}}_A} \mathbf{E},  
	\end{aligned}
	\end{equation}
	where 
	$\mathbf{H}$ is obtained by stacking all matrices $\mathbf{H}^i=
	\left[      \begin{array}{ccccc}
	\mathbf{0}_{3, 3} &  \mathbf{I}_3 &  \mathbf{0}_{3, 3(N+i-1)}&  -\mathbf{I}_3  &  \mathbf{0}_{3, 3(m-i)}
	\end{array}
	\right ] $.
	When $l$ tends to infinity, the covariance matrix becomes $\mathbf{P}^{B}_{\infty}= \mathbf{P}^A_{\infty}$ under the condition that there are three landmarks  non-coplanar with the robot position. 	
\end{theorem}
\begin{proof}
	See Appendix \ref{app::theorem3}.
\end{proof}

{ {
From Theorem 2, one can see that the estimate of RI-EKF follows the expectation that  
``the only effect
of the observations made at point B is to reduce the additional
robot uncertainty generated from the odometry noise. The
observations made at point B cannot reduce the uncertainty of
the landmark further if the robot had already observed the landmark
many times at point A. \cite{HuangT-RO2007}"
}}

We illustrate the results of Theorem 2 using the following scenario. Initially the robot is stationary at point A and continually observes ten previously unseen landmarks. It moves one step to point B  after 200 such observations and then remains stationary for 200 more time steps while observing the same set of landmarks. The initial covariance matrix of robot pose is set as non-zero.  In Fig. \ref{fig::move}, we adopt $ \log (\det (\mathbf{P}_r))$ as the extent of the uncertainty w.r.t.  robot pose, where $\mathbf{P}_r \in \mathbb{R}^{6\times 6}$ is the covariance matrix of  the robot pose. In Fig.  \ref{fig::move},  the pose uncertainty from RI-EKF  remains unchanged  and increases at time 200 when robot moves one step due to odometry noise as expected.  Further landmark observations at point B while remaining stationary gradually reduce the pose uncertainty. In contrast, the pose uncertainty from $\mathbb{SO}(3)$-EKF falls below the initial value indicating incorrect injection of information, leading to an overconfident estimate of uncertainty.

\begin{figure}[htbp]
	\centering
	\includegraphics[width=2.4in]{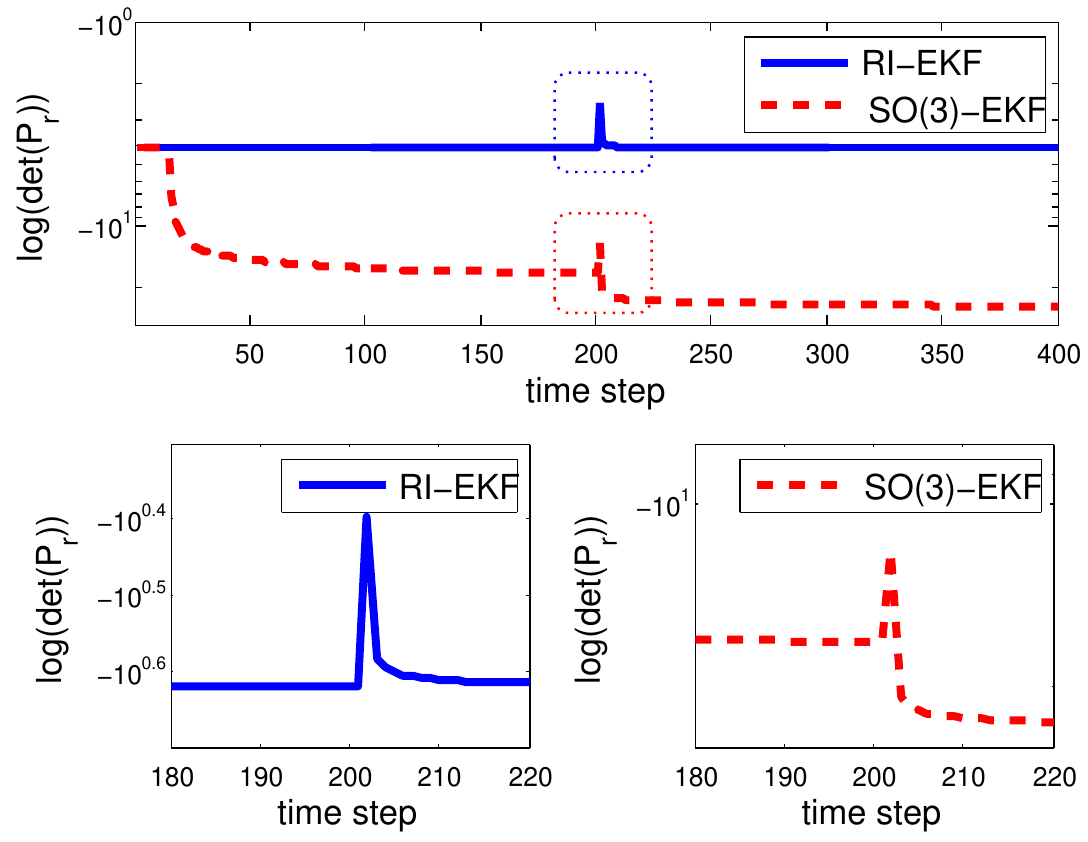}
	\caption{ \textbf{Illustration of Theorem 2}. The y-axis is $\log(\det(\mathbf{P}_r))$ that represents the pose uncertainty.  $\mathbf{P}_r$ is the covariance matrix of robot pose. Robot remains stationary from  time 1 to time 200,  moves one step at  time  200 and then remains stationary.}
	\label{fig::move}
	\vspace*{-4mm}
\end{figure}

{ {
\section{Consistency Analysis}
\label{Section::Consis}

As seen in the previous section, RI-EKF SLAM algorithm  meets the expectation that observing new landmarks does not help in reducing the robot pose uncertainty \cite{bailey2006consistency}\cite{GAS-SLAM-TRO2013}, while  $\mathbb{SO}(3)$-EKF contradicts this. 
This section further investigates the  reason for the phenomenon above. 

\subsection{Unobservability and invariance property}

This subsection first reviews the unobservability of SLAM formulation (\ref{eq::X})--(\ref{eq::obsmodel}), which is strongly related to the consistency issues of various SLAM estimation algorithms. The earliest concept of observability for nonlinear systems is proposed in \cite{hermann1977nonlinear}.  From the viewpoint of nonlinear systems, SLAM formulation (as a system for the actual state $\mathbf{X}$) is not  locally observable \cite{hermann1977nonlinear}, as understood in \cite{GPHuang-consistencyIJRR}\cite{hesch2013camera}. 
In the following, we will mathematically describe the unobservability of   SLAM formulation (\ref{eq::X})--(\ref{eq::obsmodel}) in terms of 
stochastic rigid body transformation. 
\begin{definition}
	\label{def:un}
	For SLAM problem formulation (\ref{eq::X})--(\ref{eq::obsmodel}), a stochastic rigid body transformation $\mathcal{T}_{\mathbf{g}}$  is 
     \begin{equation}
	\begin{aligned}
	\mathcal{T}_{\mathbf{g}} &(\mathbf{X})=( \exp(\Theta_1) \bar{\mathbf{R} }  \mathbf{R}, \exp(\Theta_1) \bar{\mathbf{R} }  \mathbf{p}+\bar{\mathbf{T}} +\Theta_2   ,\\
	&  \exp(\Theta_1) \bar{\mathbf{R} }  \mathbf{f}^1+\bar{\mathbf{T}} +\Theta_2 ,
	\cdots,  \exp(\Theta_1) \bar{\mathbf{R} }  \mathbf{f}^N+\bar{\mathbf{T}} +\Theta_2  ),
	\end{aligned}
	\label{eq::st}
	\end{equation}	
	 where $\mathbf{X}$ is given in (\ref{eq::X}),  $\mathbf{g}= (\bar{\mathbf{R}} , \bar{\mathbf{T}}, \Theta  )$,  $ \bar{\mathbf{R}} \in \mathbb{SO}(3) ,\bar{\mathbf{T}} \in \mathbb{R}^3$ and  $\Theta= \left[  \begin{array}{cc}
	 \Theta_1^\intercal & \Theta_2^\intercal
	 \end{array} \right]^\intercal \in \mathbb{R}^6 $ is white Gaussian noise with covariance $\bar{\pmb{\Sigma}}$. When the covariance $\bar{\pmb{\Sigma}}=\mathbf{0}_{6,6}$, this transformation degenerates into a \textbf{deterministic rigid body transformation}. When 
	 ${\mathbf{g}}=(\mathbf{I}_3,\mathbf{0}_{3,1}, \Theta )$, this transformation degenerates into a \textbf{stochastic identity transformation}.
\end{definition}

It can be easily verified  that \textbf{the output (observations)  of the system} (\ref{eq::X})--(\ref{eq::obsmodel}) is invariant to any stochastic rigid body transformation $\mathcal{T}_{\mathbf{g}}$: for any two initial  conditions,  $\mathbf{X}_0$ and $\mathbf{Y}_0:=\mathcal{T}_{\mathbf{g}}( \mathbf{X}_0)$,  we have $h_{n}(\mathbf{X}_{n}, \pmb{\xi}_{n})=h_{n}(\mathbf{Y}_{n}, \pmb{\xi}_{n}) $ for all $n\geq 0$, where $\mathbf{X}_{k} = f(\mathbf{X}_{k-1}, \mathbf{u}_{k-1}, \pmb{\epsilon}_{k-1}) $ and $\mathbf{Y}_{k} = f(\mathbf{Y}_{k-1}, \mathbf{u}_{k-1}, \pmb{\epsilon}_{k-1}) $ $(k=1,\cdots,n-1)$. Therefore, SLAM formulation (\ref{eq::X})--(\ref{eq::obsmodel}) is \textit{unobservable} in terms of stochastic rigid body transformation. 
In the following,  the invariance to stochastic rigid body transformation for the EKF framework based filter of SLAM formulation will be mathematically described.
\begin{definition}
The output (\textbf{estimated observations}) of a general EKF framework based filter is invariant under any stochastic rigid body transformation $ \mathcal{T}_{{\mathbf{g}}}$ if  for 
any  two initial estimates $(\hat{\mathbf{X}}_0, \mathbf{P}_0)$ and  $ ( \hat{\mathbf{Y}}_0, \mathbf{P}y_0 ) $, where  $ \hat{\mathbf{Y}}_0= \mathcal{T}_{\mathbf{g}} (  \hat{\mathbf{X}}_0 ) $ and $ \mathbf{P}y_0=\bar{\mathbf{Q}}_1\mathbf{P}_0\bar{\mathbf{Q}}^\intercal_1+ \bar{\mathbf{Q}}_2 \bar{\pmb{\Sigma}}\bar{\mathbf{Q}}^\intercal_2  $ in which  \begin{equation}
\begin{aligned}
\bar{\mathbf{Q}}_1&=\left.\frac{\partial  \mathcal{T}_{\hat{\mathbf{g}}} (\hat{\mathbf{X}}_0 \oplus \mathbf{e}  ) \ominus  \mathcal{T}_{\hat{\mathbf{g}}} (\hat{\mathbf{X}}_0 )  }{\partial \mathbf{e}}\right\rvert_{\mathbf{e}=\mathbf{0} },  \\
\bar{\mathbf{Q}}_2&=\left.\frac{\partial  \mathcal{T}_{\mathbf{g}} (\hat{\mathbf{X}}_0 ) \ominus  \mathcal{T}_{\hat{\mathbf{g}}} (\hat{\mathbf{X}}_0 )  }{\partial \Theta  }\right\rvert_{\Theta=\mathbf{0} }, \\
\end{aligned}
\end{equation}
and $\hat{\mathbf{g}}= (\bar{\mathbf{R}} , \bar{\mathbf{T}}, \mathbf{0} )$,  
	 we have $h_{n}(\hat{\mathbf{X}}_{n},\mathbf{0})=h_{n}(\hat{\mathbf{Y}}_{n},\mathbf{0})$ for all $n > 0$. The notations  
$\hat{\mathbf{X}}_{n}$ and $\hat{\mathbf{Y}}_{n}$ above represent the \textit{mean} estimate of this  filter at time $n$
by using the same input (odometry and observations) from time $0$ to $n$,  from  the initial conditions $(\hat{\mathbf{X}}_0, \mathbf{P}_0)$ and $(\hat{\mathbf{Y}}_0, \mathbf{P}y_0)$,  respectively. 
\end{definition}

As shown in Def. 1 and Def. 2,  the invariance to stochastic rigid body transformation can be divided into two properties: 1) \textbf{the invariance to deterministic rigid body transformation} and 2) \textbf{the invariance to stochastic identity transformation}.
The results about the invariance of some general EKF framework based filters are summarized in Theorem 3 and TABLE \ref{table::inv}.  
 \begin{theorem}
 	The output of RI-EKF is invariant under stochastic rigid body transformation. 
 	The output of $\mathbb{SO}(3)$-EKF is only invariant under deterministic rigid body transformation.
 	The output of Pseudo-RI-EKF is only invariant under stochastic identity transformation.
 	The output of $\mathbb{SE}(3)$-EKF is \textit{not} invariant under deterministic rigid body transformation or stochastic identity transformation.
\end{theorem}
\begin{proof}
	See Appendix \ref{app:theorem5}.
\end{proof}
\begin{remark}
	From the proof of Theorem 3, one can see that 
the uncertainty representation of RI-EKF can be linearly (seamlessly) transformed under a deterministic rigid body transformation, which makes  RI-EKF  invariant under deterministic rigid body transformation. In addition, we also show that the invariance property to stochastic identity transformation directly depends on the Jacobians $\mathbf{F}_n$ and  $\mathbf{H}_n$. 
\end{remark}

 	\begin{table}[htbp]
 		\centering
 		\caption{\textbf{The invariance property of the general EKF framework based filters.} DRBT/SRBT is short for ``deterministic/stochastic rigid body transformation" and SIT is short for ``stochastic identity transformation".}
 		\begin{tabular}{|l| r|r|r|}
 			\hline
 			Filters &DRBT&SIT &SRBT  \\
 			\hline
 			RI-EKF	&  Yes &Yes&Yes    \\ 	
 			Pseudo-RI-EKF	&  No&Yes&No   \\ 	
 			$\mathbb{SO}(3)$-EKF & Yes  & No& No  \\	
 		 			$\mathbb{SE}(3)$-EKF & No  & No & No  \\			
 			\hline
 		\end{tabular}
 		\label{table::inv}
 		\vspace*{-5mm}
 	\end{table}
 

\subsection{Consistency and invariance}
\label{Section::criterion}

The unobservability  in terms of stochastic rigid body transformation is a fundamental property of SLAM formulation. 
Therefore a consistent filter (as a system for the estimated state $\hat{\mathbf{X}}$) should maintain this unobservability, i.e.,  {\bf the (estimated) output of the estimator is invariant under any stochastic rigid body transformation.}  Essentially speaking, the invariance to deterministic rigid body transformation  can be interpreted that the estimate does \textit{not} depend on the selection of the global frame and the invariance to stochastic identity transformation can be understood that the uncertainty w.r.t the global frame does not affect the estimate.  
Hence, \textbf{consistency for the general EKF framework based filter is tightly coupled with the invariance to stochastic rigid body transformation. }
If a filter does not have this property, then unexpected information will be generated by the selection of the global frame or the uncertainty w.r.t. the global frame, which results in inconsistency (overconfidence).
One can see that $\mathbb{SO}(3)$-EKF, not invariant to stochastic identity transformation, produces clearly illogical estimate (the pose uncertainty is reduced by the new landmarks) in the two cases of Section \ref{Section::Convergence} while RI-EKF, 
invariant to stochastic rigid body transformation, produces the expected estimate.

\begin{remark}
	\label{remark::FEJ}
	In \cite{GPHuang-consistency08} \cite{GPHuang-consistencyIJRR}, a framework for designing an observability constrained filter is proposed. 
	 The keypoint of the observability constrained filter  is evaluating
 the Jacobians $\mathbf{F}_i$ and $\mathbf{H}_{i}$ ($i\geq 0$) at some selected points (instead of the latest estimate).
	In this way,  the output of the filter would be invariant under the stochastic identity transformation.
	On the other hand, this filter models the state space as $\mathbb{SO}(2) \times \mathbb{R}^{2+2N} $ and hence the output is  invariant under deterministic rigid body transformation (see the property of $\mathbb{SO}(3)$-EKF shown in Theorem 3). Finally, the resulting  filter indeed has the invariance property to stochastic rigid body transformation.
\end{remark}
\begin{remark}
	In  \cite{BarrauB15} the  observability analysis is  performed	on the linearized error-state model from the viewpoint of information matrix. 
	Our insight is in a different viewpoint that an estimator should mimic the unobservability (to stochastic rigid transformation) of the original system, which makes our analysis more intuitive and general. 
\end{remark}
}}

\section{Simulation Results}
\label{Section::Simulations}

In order to validate the theoretical results, we perform Monte Carlo simulations and compare RI-EKF to $\mathbb{SO}(3)$-EKF, Robotcentric-EKF, the First Estimates Jacobian EKF SLAM algorithm (FEJ-EKF), Pseudo-RI-EKF and 
$\mathbb{SE}(3)$-EKF   under conditions of different noise levels. The original Robocentric-EKF and FEJ-EKF are  proposed in 2D SLAM. For comparison, we extend these into 3D.

\subsection{Settings}
\label{Section::Performance Evaluation}

Consider that a robot moves in a  trajectory (contained in a $50m \times 40m \times 20m$ cubic)  
which allows sufficient 6-DOFs motion. In this environment, 300 landmarks are randomly generated around the specified robot trajectory.  The observations and odometry with noises are randomly generated by this specific trajectory and the simulated robot always observes the landmarks in the sensor range (less than 20m and $120^\circ$FoV). In every simulation, the number of steps is  500 (about 8 loops), the landmarks are incrementally added into the state vector and the initial covariance of robot is set as zero matrix.
For each condition (different noise level), 100 Monte Carlo simulations are performed. 
 The simulation results are summarized in Fig. \ref{fig::nees} and Table \ref{table::3D0}, where  $\sigma_{od}$ is the odometry noise level and $\sigma_{ob}$ is the observation noise level such that the covariance matrices of odometry and observation and  is $\pmb{\Phi}_n=\sigma_{od}^2\text{diag}( | \mathbf{u}_{n,1}|^2,\cdots, |\mathbf{u}_{n,6}|^2)  $ and $\pmb{\Psi}^i_{n}= \sigma_{ob}^2 \text{diag}( | \mathbf{Z}^i_{n,1}|^2,\cdots, |\mathbf{Z}^i_{n,3}|^2) $, where $ \mathbf{Z}^i_n= [ \mathbf{Z}^i_{n,1},\mathbf{Z}^i_{n,2} , \mathbf{Z}^i_{n,3} ]^\intercal =\mathbf{R}^\intercal_n ( \mathbf{f}^i-\mathbf{p}_n) $ is the ground truth of the coordinates of landmark $i$ relative to the robot pose $n$. The root mean square (RMS) error and the average normalized estimation error squared (NEES) are used to evaluate  accuracy and consistency, respectively. 

{ {
\subsection{Results and analysis}

As shown in Table \ref{table::3D0}, the estimate of  $\mathbb{SE}(3)$-EKF diverges even under the condition of low noise ($\sigma_{od}=1\%$, $\sigma_{ob}=1\%$) and  Pseudo-RI-EKF is also poor performing.
 These results can be understood because $\mathbb{SE}(3)$-EKF has no invariance property to 
deterministic rigid body transformation or stochastic identity transformation
and Pseudo-RI-EKF is not invariant under deterministic rigid body transformation, which are proven in Theorem 3. 
$\mathbb{SO}(3)$-EKF, not invariant to stochastic identity transformation,
 is also not good performing in terms of consistency. 

An interesting point in  Table \ref{table::3D0} is the performance of Robocentric-EKF.  The uncertainty representation w.r.t landmarks in Robocentric-EKF refers to the uncertainty 
of the coordinates of landmarks relative to robot frame. In this way, Robocentric-EKF has the invariance property to stochastic rigid body transformation. However,  Robocentric-EKF does not perform well under the condition of high noise ($\sigma_{od}=5\%$, $\sigma_{ob}=5\%$) because it incurs greater linearization errors in the propagation step due to the coordinate transformation applied to the landmarks, as compared to $\mathbb{SO}(3)$-EKF, FEJ-EKF and RI-EKF. 

RI-EKF, invariant to stochastic rigid body transformation, is  the best performing filter  as shown in Table \ref{table::3D0} and it is also consistent  in terms of the $95\%$ confidence bound as shown in Fig. \ref{fig::nees}. 
 Similar results for 2D cases have been reported in \cite{BarrauB15}.
 On the other hand, it is discussed in  \textit{Remark} \ref{remark::FEJ} of Section V that FEJ-EKF  also has the invariance property to stochastic
rigid body transformation but it performs less well than RI-EKF. It  can be explained that
FEJ-EKF  uses a less accurate
estimate as linearization point for evaluating the Jacobians while RI-EKF can always safely
employ the latest estimate in Jacobians.

}}

\begin{table*}[htbp]
	\centering
	\caption{\textbf{Performance Evaluation}}
	\begin{tabular}{|ccccccc|}
			\hline
			$\sigma_{od}=1\%$, $\sigma_{ob}=1\%$	&   RI-EKF &  FEJ-EKF & $\mathbb{SO}(3)$-EKF & Robocentric-EKF & Pseudo-RI-EKF & $\mathbb{SE}(3)$-EKF   \\ \hline		
			RMS of position(m) & \textbf{0.25} & 0.29 & 0.32  & 0.31 & 0.65 & Diverge      \\
			RMS of orientation(rad) & \textbf{0.0058} & 0.0071 &   0.0065 & 0.0060  & 0.0081 & Diverge   \\
			NEES of orientation & \textbf{1.02}     & 1.12     &  1.34  &  1.04 & 2.91 & Diverge   \\
			NEES of pose  & \textbf{1.01}  &   1.14      &    1.35  & 1.15 & 10 & Diverge   \\
					\hline
	$\sigma_{od}=5\%$, $\sigma_{ob}=5\%$	&   RI-EKF &  FEJ-EKF & $\mathbb{SO}(3)$-EKF & Robocentric-EKF & Pseudo-RI-EKF & $\mathbb{SE}(3)$-EKF   \\ \hline		
	RMS of position(m) & \textbf{1.16} & 1.24 & 2.0  & 2.4 &  3.90 & Diverge  \\
	RMS of orientation(rad) & \textbf{0.027} & 0.029  &   0.043 & 0.041 &  0.041 & Diverge \\
	NEES of orientation & \textbf{1.0}     & 1.05     &  3.7 & 3.0 &  1.77 & Diverge   \\
	NEES of pose  & \textbf{1.01}  &  1.13      &   3.1  & 7.5 &  92 & Diverge  \\
	\hline		
	\end{tabular}
	\vspace*{-5mm}
	\label{table::3D0}
\end{table*}

\begin{figure}[h]
	\begin{center}$
		\begin{array}{cc}
		\includegraphics[width=1.5in]{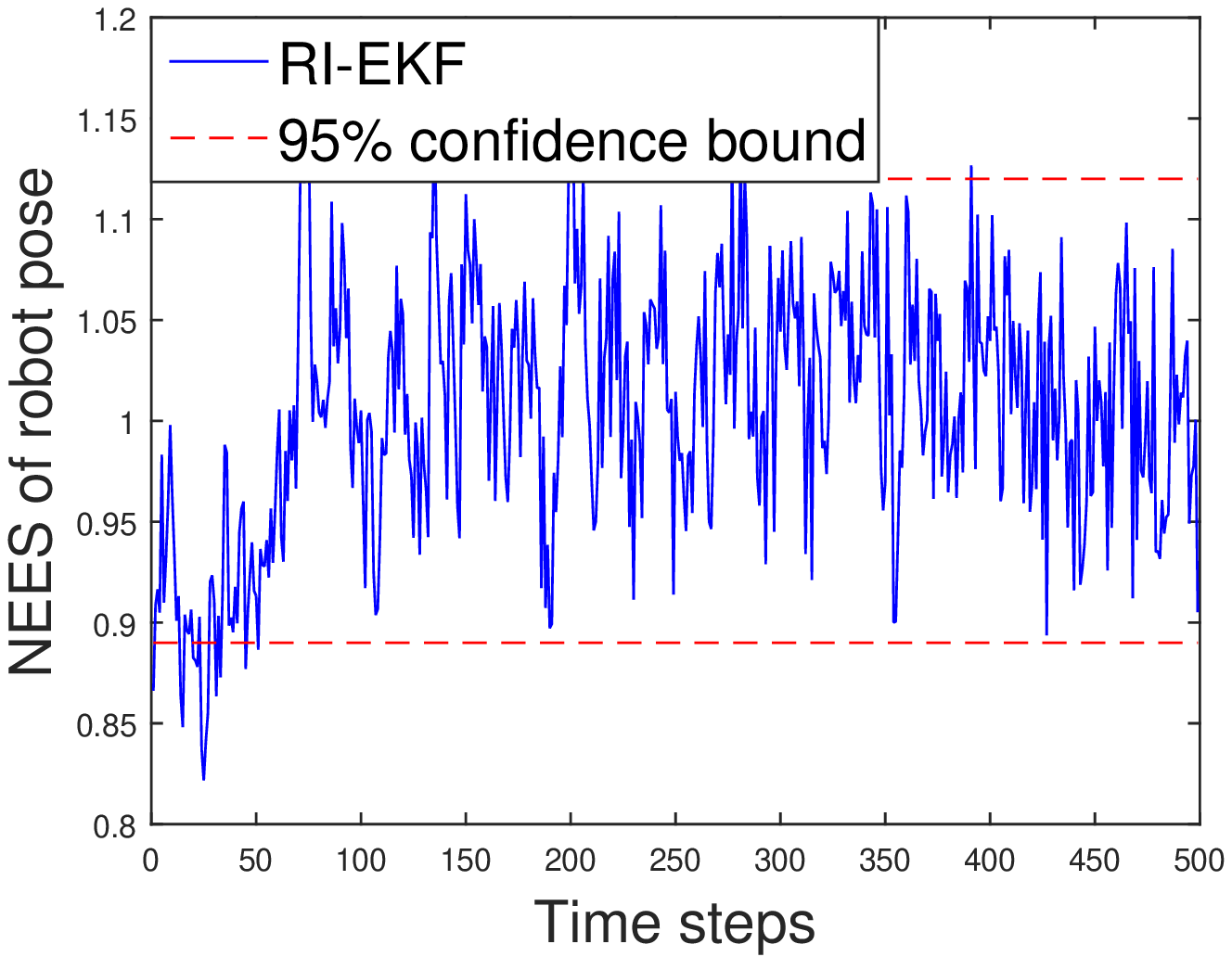}&
		\includegraphics[width=1.5in]{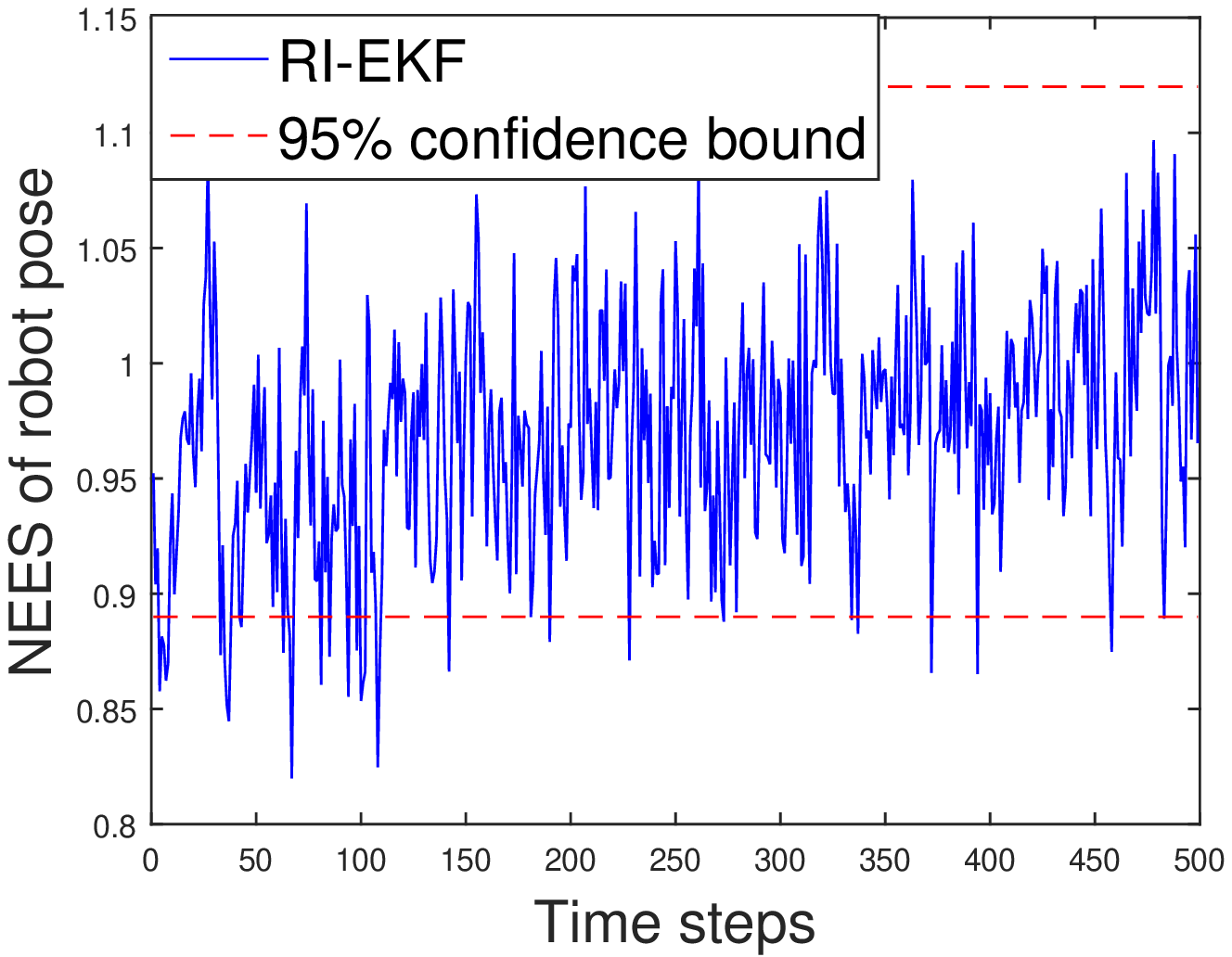}
		\end{array}$
	\end{center}
	\caption{\textbf{Average NEES of robot pose by RI-EKF from 100 Monte Carlo results. } The $95\%$ confidence bound is $[0.89, 1.12 ]$.
		 Left: $\sigma_{od}=1\%$, $\sigma_{ob}=1\%$.  Right: $\sigma_{od}=5\%$, $\sigma_{ob}=5\%$.}
	\label{fig::nees}
		\vspace*{-6mm}
\end{figure}

\section{Conclusion}
\label{Section::Conclusion}
In this work,  the convergence properties and consistency of a Lie group based invariant-EKF SLAM algorithm (RI-EKF) are analyzed.  For convergence,  several theorems with proofs are provided for two fundamental cases.
 For consistency, we propose that consistency of the general EKF framework based filter is tightly coupled with the invariance property. We also proves that the output of RI-EKF is invariant under stochastic rigid body transformation while the output of $\mathbb{SO}(3)$-EKF is only invariant under deterministic rigid body transformation. Monte Carlo simulation results demonstrates that
 the invariance property has an important impact on the consistency and accuracy of the estimator  and
  RI-EKF outperforms $\mathbb{SO}(3)$-EKF, $\mathbb{SE}(3)$-EKF, Robocentric-EKF and FEJ-EKF for 3D SLAM. 
  Future work includes extensively comparing the performance of RI-EKF SLAM algorithm with the optimization based SLAM algorithms to identify situations under which RI-EKF is sufficient, as well as extending RI-EKF to the case of  visual-inertial fusion.  
 

\appendix

\subsection{Lie Group $\mathcal{G}(N)$}\label{app::liegroup}
The notation $\mathcal{G}(N)$ is a Lie group, defined as
\begin{equation}
\mathbf{G}=\{ \left( \mathbf{R}, \mathbf{p}, \mathbf{f}^1,\cdots, \mathbf{f}^N\right) | \mathbf{R} \in \mathbb{SO}(3) , \mathbf{p} \text{\ and\ }\mathbf{f}^i\in\mathbb{R}^3    \}.
\end{equation} 
The associated group operation of $\mathcal{G}(N)$ is 
\begin{equation}
\begin{aligned}
&\mathbf{X}_1 \mathbf{X}_2= \left( \mathbf{R}_1 \mathbf{R}_2, \mathbf{R}_1\mathbf{p}_2+\mathbf{p}_1,  \cdots, \mathbf{R}_1\mathbf{f}^N_2+\mathbf{f}^N_1 \right),
\end{aligned}
\end{equation} 
where $ \mathbf{X}_{i}=\left (\mathbf{R}_{i}, \mathbf{p}_{i},\mathbf{f}^1_{i},\cdots, \mathbf{f}^N_{i} \right ) \in \mathcal{G}(N)  $ for $i = 1,2$. The associated Lie algebra of $\mathcal{G}(N)$ is homomorphic to $\mathbb{R}^{3N+6}$. The exponential mapping $\exp$ is represented as
\begin{equation}
\begin{aligned}
&\exp(\mathbf{\mathbf{e}})\in \mathcal{G}(N) \\
=&\left( \exp(\mathbf{e}_\theta), {J}_r(-\mathbf{e}_\theta)\mathbf{e}_p, {J}_r(-\mathbf{e}_\theta)\mathbf{e}^{1} , \cdots,    {J}_r(-\mathbf{e}_\theta)\mathbf{e}^N \right)   \\
\end{aligned}
\label{eq::exp}
\end{equation}
for $\mathbf{e} = \left[ \begin{array}{ccccc}
\mathbf{e}^\intercal_\theta &  \mathbf{e}_p^\intercal  & (\mathbf{e}^1)^\intercal & \cdots & (\mathbf{e}^N)^\intercal
\end{array}\right]^\intercal  \in \mathbb{R}^{3N+6}$, where $\mathbf{e}_\theta$, $\mathbf{e}_p$ and $\mathbf{e}^i \in \mathbb{R}^3$ ($i=1,\cdots,N$), the notation $\exp$ in the right side of (\ref{eq::exp}) and the mapping $\text{J}_r$ are given:  \begin{equation}
\begin{aligned}
&\exp(\mathbf{y})  =\mathbf{I}_3+\frac{\sin (\| \mathbf{y} \| )}{\| \mathbf{y} \|} {S}(\mathbf{y})+ \frac{1-\cos(\| \mathbf{y} \|) }{\| \mathbf{y} \|^2}{S}^2(\mathbf{y})   \\
\end{aligned}
\label{eq::expS}
\end{equation}
\begin{equation}
	{J}_r(\mathbf{y}) =\mathbf{I}_3 -\frac{1-\cos (\| \mathbf{y} \| )}{\| \mathbf{y} \|^2} {S}(\mathbf{y})+ \frac{\| \mathbf{y} \|-\sin(\| \mathbf{y} \|) }{\| \mathbf{y} \|^3}{S}^2(\mathbf{y})
\end{equation}
for $\mathbf{y}\in\mathbb{R}^3$. 
The adjoint  $\text{ad}_\mathbf{X}$ is computed as 
\begin{equation}
\text{ad}_\mathbf{X} = \left[
\begin{array}{ccccc}
\mathbf{R} &  \mathbf{0}_{3,3} & \cdots & \cdots & \mathbf{0}_{3,3} \\
{S}(\mathbf{p}) \mathbf{R}    & \mathbf{R}   & \ddots  &  & \vdots \\
{S}(\mathbf{f}^1) \mathbf{R}   & \mathbf{0}_{3,3} & \mathbf{R} & \ddots & \vdots    \\      
\vdots & \vdots  & \ddots  & \ddots  &   \mathbf{0}_{3,3} \\
{S}(\mathbf{f}^N)\mathbf{R} & \mathbf{0}_{3,3}  & \cdots & \mathbf{0}_{3,3}   & \mathbf{R}
\end{array}
\right].
\end{equation}

\subsection{Proof of Theorem 1}\label{app::theorem2}
In the following, we  use  mathematical induction to prove this theorem. Note that
At the beginning, the estimate is $(\hat{\mathbf{X}},\mathbf{P})$ where
$\hat{\mathbf{X}}=
\left(               
\hat{\mathbf{R}} , \hat{\mathbf{p}}, \hat{\mathbf{f}}^1,\cdots,\hat{\mathbf{f}}^N   \\
\right )               
$.
After the first observation, the mean estimate of state and covariance matrix are augmented as below via the method shown in  Alg. \ref{Alg::Aug}:  $\hat{\mathbf{X}}_1 =
\left(               
\hat{\mathbf{R}} , \hat{\mathbf{p}} , \hat{\mathbf{f}}^1,\cdots,\hat{\mathbf{f}}^N ,  \hat{\mathbf{f}}^{N+1}   \\  
\right )$ and 
$ \mathbf{P}_{1} = \left[\begin{array}{cc}
\mathbf{P} &   \mathbf{L}    \\
\mathbf{L}^\intercal & \hat{\mathbf{R}} \pmb{\Psi} \hat{\mathbf{R}}^\intercal + \mathbf{W}
\end{array}
\right]  $.
Obviously, after one observation, the mean estimate of robot pose and the previous ``landmarks" does not change and the covariance matrix follows the proposed form. We now assume that after $k$ times observations, the  estimate becomes  
$\hat{\mathbf{X}}_k=
\left(              
\hat{\mathbf{R}} , \hat{\mathbf{p}} ,\hat{\mathbf{f}}^1,\cdots,\hat{\mathbf{f}}^N,  \hat{\mathbf{f}}^{N+1}_k   \\  
\right)               
$ and 
$\mathbf{P}_{k} =  \left[\begin{array}{cc}
\mathbf{P} &   \mathbf{L}    \\
\mathbf{L}^\intercal &  \frac{\hat{\mathbf{R}} \pmb{\Psi} \hat{\mathbf{R}}^\intercal}{k} + \mathbf{W}
\end{array}
\right]$.
Now we discuss the case after $k$ times observations of next propagation and update. Because the robot is always perfectly stationary, after propagation at time $k$, the mean estimate is $ \hat{\mathbf{X}}_{k+1|k}= \hat{\mathbf{X}}_{k}  $ and covariance matrix becomes 
 $ \mathbf{P}_{k+1|k} =  \mathbf{P}_{k} $.
According to Alg. \ref{Alg::G-EKF},  we have  $ \mathbf{S}=  \mathbf{H} \mathbf{P}_{k+1|k}  \mathbf{H}^\intercal+ \pmb{\Psi} = \frac{k+1}{k} \pmb{\Psi}$ and
$
	\mathbf{K}=  \mathbf{P}_{k+1|k}\mathbf{H}^\intercal\mathbf{S}^{-1} = \left[\begin{array}{cc}
	\mathbf{0}_{3,(3N+6)}  &
	-\frac{ 1   }{k+1} \hat{\mathbf{R}}^\intercal
	\end{array} \right]^\intercal,  
$
where $\mathbf{H}=
\left[               
\begin{array}{cccccc}   
\mathbf{0}_{3, 3} & \hat{\mathbf{R}}^\intercal & \mathbf{0}_{3, 3N}&  -\hat{\mathbf{R}}^\intercal    \\  
\end{array}
\right ]               
$. Then it is easy to see that all elements from the vector $\mathbf{Ky}$ are zero except the last 3 elements, and hence the estimate of robot pose and the old landmarks after $k+1$ times observations are \textbf{the same} as that  in the time step $k$. The covariance matrix at time $k+1$ is
$
\mathbf{P}_{k+1} =  (\mathbf{I}-\mathbf{KH})\mathbf{P}_{k+1|k}  = \left[\begin{array}{cc}
\mathbf{P} &   \mathbf{L}    \\
\mathbf{L}^\intercal  &  \frac{\hat{\mathbf{R}} \pmb{\Psi} \hat{\mathbf{R}}^\intercal}{k+1} + \mathbf{W}
\end{array}
\right]
$.
When $k$ converges to infinity, we have (\ref{eq::PAend}).

\subsection{Proof of Theorem 2}\label{app::theorem3}
By using result in Theorem 1 and the Jacobian matrices in (\ref{eq::FNGNL}), we have 
\begin{equation}
\mathbf{P}^0_B=\mathbf{P}^{\infty}_A + \Delta \mathbf{P},
\label{eq::A2}
\end{equation}
where 
$\mathbf{P}^{\infty}_A$ (given in (\ref{eq::MultipleAdd})) is the covariance matrix  before moving to the point $B$,
$
\Delta \mathbf{P}= {ad}_{\hat{\mathbf{X}}_A} \mathbf{E} \tilde{\pmb{\Phi}} \mathbf{E}^\intercal {ad}_{\hat{\mathbf{X}}_A}^\intercal 
$ can be regarded as the incremental uncertainty  caused by the odometry noise,
and $\tilde{\pmb{\Phi}}=   \mathbf{B} \pmb{\Phi} \mathbf{B}^\intercal  $ is a positive definite matrix.

After $l$ observations at point B, the information matrix $\pmb{\Omega}^{B}_l$ (the inverse of $\mathbf{P}^B_l$)    becomes
$
\pmb{\Omega}^{B}_l=\pmb{\Omega}^{B}_0+\sum_{j=1}^{l}\mathbf{H}_j^\intercal \bar{\pmb{\Psi}}^{-1} \mathbf{H}_j
$,
where $ \mathbf{H}_j $ is obtained by stacking all matrices 
$
\mathbf{H}^i_j=\hat{\mathbf{R}}_j^\intercal  \mathbf{H}^i
$ ($i=1,\cdots,m$),	
and $\hat{\mathbf{R}}_j$ is the estimated orientation after $j$ times observations at point B.
Note that $\bar{\pmb{\Psi}}$ is  isotropic, we have $\mathbf{H}_j^\intercal \bar{\pmb{\Psi}}^{-1} \mathbf{H}_j=\mathbf{H}^\intercal \bar{\pmb{\Psi}}^{-1} \mathbf{H}$ ($j=1,\cdots,l$). Therefore, the information matrix is $ \pmb{\Omega}^{B}_l=\pmb{\Omega}^{B}_0+ l  \mathbf{H}^\intercal \bar{\pmb{\Psi}}^{-1} \mathbf{H} $.
Via the matrix inversion lemma in \cite{HuangT-RO2007}, the covariance matrix after $l$ observations at point B is 
\begin{equation}
\begin{aligned}
\mathbf{P}^l_B & =(\pmb{\Omega}^{B}_l)^{-1} = \mathbf{P}^0_B -\mathbf{P}^0_B \mathbf{H}^\intercal  ( \frac{\bar{\pmb{\Psi}}}{l}+ \mathbf{H} \mathbf{P}^0_B  \mathbf{H}^\intercal   )^{-1}      \mathbf{H}\mathbf{P}^0_B .\\
\end{aligned}
\label{eq::A1}
\end{equation}
Note that $\mathbf{H}\mathbf{P}^{\infty}_A= \mathbf{0}$, we substitute 
(\ref{eq::A2}) into (\ref{eq::A1}):
\begin{equation}
\begin{aligned}
\mathbf{P}^l_B & =\mathbf{P}^{\infty}_A +\Delta \mathbf{P}  -\Delta \mathbf{P} \mathbf{H}^\intercal  ( \frac{\bar{\pmb{\Psi}}}{l}+ \mathbf{H} \Delta \mathbf{P}  \mathbf{H}^\intercal   )^{-1}      \mathbf{H}\Delta \mathbf{P} \\
&= \mathbf{P}^{\infty}_A + {ad}_{\hat{\mathbf{X}}_A} \mathbf{E} ( \tilde{\pmb{\Phi}}^{-1}+l\tilde{\mathbf{H}}^\intercal \bar{\pmb{\Psi}}^{-1}    \tilde{\mathbf{H}}   )^{-1}  \mathbf{E}^\intercal {ad}_{\hat{\mathbf{X}}_A}^\intercal\\
&=\mathbf{P}^{\infty}_A + \bar{\mathbf{P}}^{l}_B.	
\end{aligned}
\end{equation}
Furthermore, 
$
\begin{aligned}
\tilde{\mathbf{H}}^\intercal \bar{\pmb{\Psi}}^{-1}    \tilde{\mathbf{H}} 
=  \left[
\begin{array}{cc}
\mathbf{S}_1  & \mathbf{S}_2 \\
\mathbf{S}_2^\intercal  & m \pmb{\Psi}^{-1}
\end{array}
\right]
\end{aligned}
$
where  $\mathbf{S}_1= \sum_{i=1}^{m}{S}^\intercal (\tilde{\mathbf{f}}_i) \pmb{\Psi}^{-1}{S}(\tilde{\mathbf{f}}_i)  $,
$ \mathbf{S}_2=   (\sum_{i=1}^{m} {S}(\tilde{\mathbf{f}}_i))^\intercal  \pmb{\Psi}^{-1}   $ and
$\tilde{\mathbf{f}}_i = \hat{\mathbf{R}}^\intercal ( \hat{\mathbf{p}} - \hat{\mathbf{f}}_i  )  $ ($i=1,\cdots,m$). Generally speaking, $ \tilde{\mathbf{H}}^\intercal \pmb{\Psi}^{-1}    \tilde{\mathbf{H}} $ is full rank when $m>3$ and there are three landmarks that are non-coplanar with the robot position. Under this condition, it is easy to see that 
$ \mathbf{P}^l_B \rightarrow \mathbf{P}^{\infty}_A  $ when $ l \rightarrow \infty$.

\subsection{Proof of Theorem 3}\label{app:theorem5}
Here, we only prove that the invariance property of RI-EKF and $\mathbb{SO}(3)$-EKF.  The invariance properties of the other algorithms can be easily proven in a similar way or through a counter example.  

First, we prove that the outputs of $\mathbb{SO}(3)$-EKF and RI-EKF is invariant to deterministic rigid body transformation.   
Assume the estimate at time $0$ is 
$(\hat{\mathbf{X}}_0,\mathbf{P}_0)$ in terms of the general EKF framework. After one step propagation via the odometry $\mathbf{u}_0$, the estimate becomes $(\hat{\mathbf{X}}_{1|0},\mathbf{P}_{1|0})$. Then after obtaining observations $\mathbf{z}_{1}$, the estimate becomes $(\mathbf{X}_{1}, \mathbf{P}_{1})$. 
On the other hand, in $\mathbb{SO}(3)$-EKF and RI-EKF, there exists a matrix $\mathbf{Q}_{\mathcal{T}}$ for any rigid body transformation $\mathcal{T}$ such that
\begin{equation}
\mathcal{T} (\mathbf{X} \oplus \mathbf{Q}_{\mathcal{T}}^{-1} \mathbf{e}  ) = \mathcal{T} (\mathbf{X}  ) \oplus  \mathbf{e}\ \  \ \forall\ \mathbf{X}.
\label{eq::add+}
\end{equation}   
Therefore, if 
 a deterministic rigid body transformation $\mathcal{T}$ is applied at time $0$, the estimate   becomes $  (\hat{\mathbf{Y}}_0,\mathbf{P}y_0) $, where $\hat{\mathbf{Y}}_0= \mathcal{T}(\hat{\mathbf{X}}_0)   $ and $\mathbf{P}y_0= 
  \mathbf{Q}_{ \mathcal{T} } \mathbf{P}_0 \mathbf{Q}_{ \mathcal{T} }^\intercal $. 
Now we calculate the new Jacobians $\mathbf{F}y_0$ and $\mathbf{G}y_0$   in propagation
 \begin{equation}
	\begin{aligned}
\mathbf{F}y_0=& \left.\frac{\partial  f ( \hat{\mathbf{Y}}_0 \oplus \mathbf{e}, \mathbf{u}_0, \mathbf{0}  ) \ominus  f ( \hat{\mathbf{Y}}_0, \mathbf{u}_0, \mathbf{0}   ) }{\partial \mathbf{e}}\right\rvert_{\mathbf{0} } \\
=&  \left. \frac{\partial  f ( \mathcal{T}( \hat{\mathbf{X}}_0) \oplus \mathbf{e}, \mathbf{u}_0, \mathbf{0}  ) \ominus  f ( \mathcal{T}( \hat{\mathbf{X}}_0), \mathbf{u}_0, \mathbf{0}   ) }{\partial \mathbf{e}}\right\rvert_{\mathbf{0} } \\
\overset{(\ref{eq::add+})}{=}& \left. \frac{\partial  f ( \mathcal{T}( \hat{\mathbf{X}}_0 \oplus \mathbf{Q}^{-1}_{\mathcal{T}} \mathbf{e}), \mathbf{u}_0, \mathbf{0}  ) \ominus  f ( \mathcal{T}( \hat{\mathbf{X}}_0), \mathbf{u}_0, \mathbf{0}   ) }{\partial \mathbf{e}}\right\rvert_{\mathbf{0} } \\
=& \left. \frac{\partial  \mathcal{T} (  f( \hat{\mathbf{X}}_0 \oplus \mathbf{Q}_{\mathcal{T}}^{-1} \mathbf{e}, \mathbf{u}_0, \mathbf{0}  )) \ominus  \mathcal{T} ( f( \hat{\mathbf{X}}_0, \mathbf{u}_0, \mathbf{0}  ) ) }{\partial \mathbf{e}}\right\rvert_{\mathbf{0} }     \\
=& \left. \frac{\partial  \mathcal{T} (  f( \hat{\mathbf{X}}_0, \mathbf{u}_0, \mathbf{0}) \oplus \mathbf{F}_{0}  \mathbf{Q}_{\mathcal{T}}^{-1} \mathbf{e})   \ominus  \mathcal{T} ( f( \hat{\mathbf{X}}_0, \mathbf{u}_0, \mathbf{0}  ) ) }{\partial \mathbf{e}}\right\rvert_{\mathbf{0} }     \\
\overset{(\ref{eq::add+})}{=}&  \mathbf{Q}_{\mathcal{T}}\mathbf{F}_{0} \mathbf{Q}_{\mathcal{T}}^{-1}.
	\end{aligned}
\end{equation}
Similarly, we have $ \mathbf{G}y_0= \mathbf{Q}_{\mathcal{T}}  \mathbf{G}_0 $. Hence,   after one step propagation the estimate becomes $(\hat{\mathbf{Y}}_{1|0},\mathbf{P}y_{1|0})$, where $\hat{\mathbf{Y}}_{1|0}=f(\hat{\mathbf{Y}}_{0},\mathbf{u}_0,\mathbf{0} ) = \mathcal{T}(\hat{\mathbf{X}}_{1|0})  $
and  $\mathbf{P}y_{1|0}= \mathbf{F}y_0 \mathbf{P}y_{0}\mathbf{F}y^\intercal_0+  \mathbf{G}y_0 \pmb{\Phi}_0 \mathbf{G}y^\intercal_{0}   = \mathbf{Q}_{\mathcal{T}} \mathbf{P}_{1|0} \mathbf{Q}^\intercal_{\mathcal{T}}    $. The new Jacobians  in update becomes $\mathbf{H}y_{1}= \mathbf{H}_{1} \mathbf{Q}^{-1}_{\mathcal{T}}  $. Then it is easy to obtain $\mathbf{K}_y=\mathbf{Q}_{\mathcal{T}}   \mathbf{K}  $, resulting in $\hat{\mathbf{Y}}_{1}= \hat{\mathbf{Y}}_{1|0}\oplus \mathbf{K}_y \mathbf{y} =\mathcal{T}( \hat{\mathbf{X}}_{1|0}  )  \oplus  \mathbf{Q}_{\mathcal{T}} \mathbf{K} \mathbf{y} = \mathcal{T} ( \hat{\mathbf{X}}_{1|0} \oplus  \mathbf{K} \mathbf{y}     )  =  \mathcal{T}( \hat{\mathbf{X}}_{1}  )   $. The covariance matrix after update  becomes 
$ \mathbf{P}y_{1}= (\mathbf{I}- \mathbf{K}_y\mathbf{H}y_{1} )\mathbf{P}y_{1|0}=    \mathbf{Q}_{\mathcal{T}} \mathbf{P}_{1}  \mathbf{Q}^\intercal_{\mathcal{T}}$.
In all, $\hat{\mathbf{Y}}_{1}= \mathcal{T}( \hat{\mathbf{X}}_{1}  )  $  and  $ \mathbf{P}y_{1}= \mathbf{Q}_{\mathcal{T}} \mathbf{P}_{1}  \mathbf{Q}^\intercal_{\mathcal{T}}   $. By mathematical induction, we can see 
 the outputs of $\mathbb{SO}(3)$-EKF (and RI-EKF) are invariant under deterministic rigid body transformation.

Secondly, we prove the invariance property of RI-EKF  under  stochastic identity body transformation $ \mathcal{T}_{\mathbf{g}}$ ($ \mathbf{g}=(\mathbf{I}_3,\mathbf{0}, \Theta   ) $) for all $\bar{\pmb{\Sigma}}$ where $\bar{\pmb{\Sigma}}$ is the covariance matrix of noise $\Theta$. Consider  the estimate at time $0$ is 
$(\hat{\mathbf{X}}_0,\mathbf{P}_0)$ in RI-EKF. 
If the stochastic rigid body transformation $ \mathcal{T}_{\mathbf{g}}$ is applied, the estimate 
becomes $  (\hat{\mathbf{X}}_0,\mathbf{P}_0+\Delta \mathbf{P} )  $ where  $\Delta \mathbf{P}= \mathbf{C} \bar{\Sigma}  \mathbf{C}^\intercal$ and
\begin{equation}
\mathbf{C}=  \frac{ \partial   \mathcal{T}_{\mathbf{g}} ( \hat{\mathbf{X}}_{0} ) \ominus      \hat{\mathbf{X}}_{0}       }{\partial \Theta}|_{\mathbf{0}}= \left[ \begin{array}{cc}
\mathbf{I}_3 & \mathbf{0}_{3,3} \\
\mathbf{0}_{3,3}  &  \mathbf{I}_3\\
\vdots  &  \vdots \\
\mathbf{0}_{3,3}  &  \mathbf{I}_3
\end{array}   \right]. 
\end{equation}
After propagation, the estimate becomes $  (\hat{\mathbf{X}}_{1|0},\mathbf{P}_{1|0}+\Delta \mathbf{P} )  $ due to $\mathbf{F}_n=\mathbf{I}$ given in (\ref{eq::FNGNL}). Note that $\mathbf{H}_{1} \Delta \mathbf{P}=\mathbf{0} $, it is easy to get the posterior estimate $ ( \hat{\mathbf{X}}_{1} , \mathbf{P}_{1}+\Delta \mathbf{P}     )   $. By mathematical induction, we can conclude that the output of RI-EKF is invariant under stochastic identity transformation.



\bibliographystyle{IEEEtran}
\bibliography{ref.bib}


\end{document}